\documentclass[journal]{IEEEtran}
\ifCLASSINFOpdf
\usepackage[pdftex]{graphicx}
\else
\fi
\usepackage{amssymb}
\usepackage{siunitx}
\usepackage{physics}
\usepackage{color}
\usepackage[ruled,linesnumbered]{algorithm2e}
\SetKwInput{KwInput}{Input}                % Set the Input
\SetKwInput{KwOutput}{Output}   
\usepackage{amsmath}
\usepackage{cite}
\usepackage{booktabs}
\usepackage{multirow}
\usepackage{amsthm}
\newtheorem{theorem}{Proposition}
\usepackage{verbatim}

\begin{document}

\title{Tiered Acquisition for Constrained Bayesian Optimization: An Application to Analog Circuits}

\author{Ria Rashid and Abhishek Gupta
\thanks{This work was supported by the Ramanujan Fellowship from the Science and Engineering Research Board, Government of India (Grant No. RJF/2022/000115).}
\thanks{Ria Rashid and Abhishek Gupta (Corresponding author) are with the School of Mechanical Sciences, Indian Institute of Technology Goa, Ponda - 403401, India (e-mail: ria183422005@iitgoa.ac.in, abhishekgupta@iitgoa.ac.in).}}

\markboth{}%
{Shell \MakeLowercase{\textit{et al.}}: Bare Demo of IEEEtran.cls for IEEE Journals}

\maketitle

\begin{abstract}
Analog circuit design can be considered as an optimization problem with the targeted circuit specifications as constraints. When stringent circuit specifications are considered, it is desired to have an optimization methodology that adapts well to heavily constrained search spaces. To this end, we propose a novel Bayesian optimization algorithm with a tiered ensemble of acquisition functions and demonstrate its considerable application potential for analog circuit design automation. Our method is the first to introduce the concept of multiple dominance among acquisition functions, allowing the search for the optimal solutions to be effectively bounded \emph{within} the predicted set of feasible solutions in a constrained search space. This has resulted in a significant reduction in constraint violations by the candidate solutions, leading to better-optimized designs within tight computational budgets. The methodology is validated in gain and area optimization of a two-stage Miller compensated operational amplifier in a \SI{65}{\nano\meter} technology. In comparison to robust baselines and state-of-the-art algorithms, this method reduces constraint violations by up to 38\% and improves the target objective by up to 43\%. The source code of our algorithm is made available at https://github.com/riarashid/TRACE.
\end{abstract}


\IEEEpeerreviewmaketitle

\section{Introduction}

\IEEEPARstart{M}{ixed} signal integrated circuits (ICs), which integrate both analog and digital circuits on a single semiconductor die, have become very popular in the present-day electronic industry. The primary obstacle to the rapid design cycles of such mixed-signal integrated circuits is the absence of automation in the analog circuit design process, which still lags behind digital circuit design \cite{11,12}. Analog circuit design can be considered as an optimization problem with the targeted circuit specifications as constraints\cite{11}. Under increasingly stringent specifications, the need for powerful constrained optimization algorithms grows due to the intractability of manually configuring circuits to satisfy all constraints.

Different optimization strategies, including evolutionary algorithms \cite{21,25}, swarm intelligence techniques like particle swarm optimization \cite{8,23} have been extensively used in analog circuit design. More recently, machine learning-based techniques \cite{28,26}, including knowledge transfer \cite{27} and artificial neural networks \cite{20}, have also found application in the design optimization process. The majority of these studies have adopted a simulation-based methodology, wherein circuit performances are viewed as black-box functions and derivative-free optimization algorithms are combined with circuit simulators for optimal designs. Computationally expensive circuit simulations remain a major challenge.

The Bayesian optimization (BO) paradigm offers a principled, general-purpose strategy to tackle such expensive simulation-based optimization problems \cite{2}, with recent applications in automated circuit designs \cite{3,4}. In \cite{1}, a weighted expected improvement-based BO approach with a Gaussian process (GP) surrogate model was proposed for the first time for analog design optimization. The idea of utilizing a multi-objective acquisition function ensemble (MACE) in BO was subsequently introduced \cite{5,6}. Acquisition functions are considered as figures of merit of a candidate solution estimated by the surrogate model. The chosen acquisition ensemble, collectively accounting for circuit optimality as well as the satisfaction of constraints, is formulated into a multi-objective optimization problem whose Pareto set yields solution candidates deemed most worthy of expensive simulation-based evaluation. However, a large ensemble could manifest in a curse of objective dimensionality \cite{31}, with points selected at random from the Pareto set likely to be constraint-violating.

To tackle the curse of dimensionality, this paper proposes to reformulate the multi-objective acquisition ensemble by splitting it into \emph{multiple dominance} levels \cite{7}. Specifically, we introduce a novel \emph{TieRed Acquisition Ensemble} (TRACE) for constrained BO where two levels of dominance have been implemented. The first level is tasked with demarcating solutions that are feasible with high likelihood, whereas the second level focuses on identifying the optimum \emph{within} the feasible set. A pair of constraints-aware novel acquisition functions have been crafted to effectively target the search inside the feasible region of a highly constrained search space. The technique is validated for the design of a two-stage Miller compensated operational amplifier (TSMCOA) in gain and area optimization in \SI{65}{\nano\meter}, demonstrating significant gains in performance compared to state-of-the-art algorithms.
\section{Preliminaries}
\subsection{Constrained Optimization Problem Formulation}
In an analog design optimization problem, the target circuit specifications are considered as the constraints and the parameters of the circuit are considered as the design variables. The objective function given by some performance metric of the circuit is, without loss of generality, to be minimized. The constrained optimization problem can then be framed as:

\begin{equation}
\begin{aligned}
& \underset{\vb*x \in \vb*X}{\min}
&& f(\vb*x) \\
& \text{subject to}
& & c_i(\vb*x) \leq 0, \forall  i \in 1,2,...C
\end{aligned}
\end{equation}
where $f(\vb*x)$ is the objective function, $\vb*x$ represents the design parameters, $c_i(\vb*x)$ is the $i^{\text{th}}$ constraint, and $C$ is the total number of constraints. $\vb*X = \{(\vb*x_1, \vb*x_2, \cdots \vb*x_d)|L_j\leq x_j\leq U_j, j=1,2\cdots d\}$ is the search space, where $d$ is the search space dimensionality, and $L_j$ and $U_j$ are the lower and upper bounds of the $j^{\text{th}}$ decision variable.

\subsection{Bayesian Optimization}
A BO framework has two main elements: (1) probabilistic surrogate models and (2) acquisition functions defined based on the models' predictions. A surrogate model acts as a computationally cheap approximation of the function, which is costly to evaluate. One of the most commonly used surrogate models in BO is the Gaussian process (GP) \cite{21}. Acquisition functions, on the other hand, ascribe figures of merit to unevaluated candidate solutions. In what follows, the GP and various acquisition functions of interest are briefly explained.
\subsubsection{Gaussian Processes (GP)}
A GP is a stochastic approach to regression that extends the concept of multivariate Gaussians to infinite dimensions \cite{32}. It places a Gaussian distribution prior over functions as $f\sim GP(m, k)$ that's completely specified by a mean function, $m(\vb*x)$, which is often set to zero, and a covariance function, $k(\vb*x,\vb*{x'})$, which correlates observations $f(\vb*x)$ and $f(\vb*x')$. If the training dataset is denoted by $D = \{\vb*X, \vb*y\}$, where $\vb*X = \{\vb*x_1, \vb*x_2,\cdots \vb*x_N\}$ is a set of \emph{N} input points, and $\vb*y = [f(\vb*x_1), f(\vb*x_2),\cdots f(\vb*x_N)]^T$ is the corresponding vector of function evaluations, then the posterior predictive mean, $\mu(\vb*x)$, and variance, $\sigma^2(\vb*x)$, of the GP at an unseen data point, $\vb*x$, is given by:
\begin{gather}
    \mu (\vb*x)= k(\vb*x, \vb*X)[K+\sigma^2_n\vb*I_N]^{-1}\vb*y\\
    \sigma^2(\vb*x)= k(\vb*x, \vb*x) - k(\vb*x, \vb*X)[K+\sigma^2_n\vb*I_N]^{-1}k(\vb*X, \vb*x)
\end{gather}
where $k(\vb*x, \vb*X)$ = $k^T(\vb*X, \vb*x)$, $K = k(\vb*X, \vb*X)$ represents the covariance matrix, $\sigma^2_n$ is an additive noise term, and $\vb*I_N$ is a $N\times N$ identity matrix. In this brief, the squared exponential function is selected as the GP covariance kernel. 
\subsubsection{Acquisition Functions}
These serve as figures of merit, which, when optimized, yield solution samples that offer a trade-off between exploration and exploitation of the search space \cite{33}. Three commonly used acquisition functions are the lower confidence bound (LCB) for minimization problems, probability of improvement (PI) and the expected improvement (EI). For a data point, $\vb*x$, their values can be computed based on the predictive mean and variance of the GP as:
\begin{gather}
    LCB(\vb*x) = \mu(\vb*x) - \beta \sigma(\vb*x)\\
    PI(\vb*x)  = \Phi(\lambda)\\
    EI(\vb*x) = \sigma(\vb*x)(\lambda \Phi(\lambda)+\phi(\lambda))
\end{gather}
where $\lambda = \frac{\tau - \epsilon - \mu({\vb*x})}{\sigma({\vb*x})}$. In these definitions, $\beta$ and $\epsilon$ are tunable parameters controlling the extent of exploration, $\tau$ is the best value among the observations in $\vb*y$, and $\Phi(.)$ and $\phi(.)$ are the cumulative distribution function and the probability density function of the standard normal distribution. In this brief, the values of $\beta$ and $\epsilon$ are set as 0.3 and 0.001, respectively.
\subsection{Optimization with Multiple Dominance Levels}
The objective of a standard (single-level) optimization problem is to search for designs that are optimal with respect to a particular binary relation. A binary relation is a comparison of two designs that determines whether one is better than the other, and a design is considered optimal if no better design exists with respect to the relation. For a multi-objective optimization problem, this binary relation is established using concepts of Pareto dominance, leading to a Pareto set of optimal solutions \cite{34}. Given a search space $\vb*X$ in which the binary Pareto dominance relation is symbolized as $\preceq_F$, the Pareto set can be expressed as:
\begin{equation}
\label{eq1}
    \vb*X_{optimal} = min(\vb*X,\preceq_F)
\end{equation}
where $F$ represents the vector of objective functions and the symbol $\preceq$ indicates that the objective functions in $F$ have to be minimized.

In TRACE, we generalize multi-objective optimization to the case of two dominance levels. The basic idea is that the second dominance level, represented by some binary relation $\preceq_{F_2}$, serves to select a subset of solutions from \emph{within} the Pareto set of solutions that are optimal with respect to a first binary relation $\preceq_{F_1}$. We denote the optimized subset as:
\begin{equation}
\label{eq2}
    \tilde{\vb*X}_{optimal} = min\big(min(\vb*X,\preceq_{F_1}),\preceq_{F_2}\big).
\end{equation}
In order to design a multi-dominance optimization algorithm that implements $\preceq_{F_1}$ and $\preceq_{F_2}$ in this nested manner, a single combined binary relation, $\preceq_{F_{1,2}}$, may be defined such that:
\begin{equation}
\label{eq3}
    min(\vb*X, \preceq_{F_{1,2}}) = min\big(min(\vb*X,\preceq_{F_1}),\preceq_{F_2}\big).
\end{equation}
Let both binary relations $\preceq_{F_1}$ and $\preceq_{F_2}$ be a strict partial order, as is necessarily the case for Pareto dominance \cite{34}. Accordingly, let the ranking of a candidate solution, $\vb*x$, under $\preceq_{F_1}$ and $\preceq_{F_2}$ be $R_1(\vb*x)$ and $R_2(\vb*x)$, respectively. A lower rank reflects a better solution under a given binary relation. The combined binary relation, $\preceq_{F_{1,2}}$, can be constructed as:
\begin{multline}
\label{eq4}
    \vb*x \preceq_{F_{1,2}} \vb*x' \iff R_1(\vb*x) < R_1(\vb*x') \cup \\ \big((R_1(\vb*x) = R_1(\vb*x')) \cap (R_2(\vb*x) < R_2(\vb*x'))\big)
\end{multline}
indicating that $\vb*x$ is better than $\vb*x'$ in the multi-dominance setting. Optimizing with multiple dominance levels thus entails finding a set, $\tilde{\vb*X}_{optimal}$, such that for each solution in $\tilde{\vb*X}_{optimal}$, no better solution exists with respect to $\preceq_{F_{1,2}}$.
\section{TRACE}
In a BO iteration, once the GP surrogates are trained and the acquisition functions defined, their optimization yields one or more candidate solutions for evaluation in the next iteration. In TRACE, we propose the ensemble of acquisition functions to be hierarchically organized. Specifically, two levels of dominance have been introduced in the acquisition function optimization stage. In the first level, we coin a pair of novel acquisition functions with the goal of helping bound the search for analog circuits within constraint-satisfying (i.e. feasible) regions of the search space. For a design optimization problem with $C$ constraints and independent GP models trained to predict the values of each constraint function, these acquisition functions are defined as:
\begin{gather}
f_{cv1}(\vb*x) = \text{max} \{\mu_{c1}(\vb*x), \mu_{c2}(\vb*x), \cdots \mu_{cC}(\vb*x)\}\label{eq5}\\
f_{cv2}(\vb*x) = \text{min} \{|\mu_{c1}(\vb*x)|, |\mu_{c2}(\vb*x)|, \cdots |\mu_{cC}(\vb*x)|\} \label{eq6}
\end{gather}
where $\mu_{ci}(\vb*x)$ represents the predictive mean of the GP corresponding to the $i^{\text{th}}$ constraint. The vectorized objective function that translates to a binary Pareto dominance relationship between solutions at the first level, i.e. $\preceq_{F_1}$, is then:
\begin{equation}
\label{eq7}
F_1 = [f_{cv1}(\vb*x), f_{cv2}(\vb*x)].
\end{equation}
Assuming accurate GPs, it can be shown that the Pareto set $min(\vb*X,\preceq_{F1})$ that results from this formulation of $F_1$ is in exact correspondence with the feasible region of $\vb*X$. A proof of this proposition is provided later in this section.

The acquisition functions LCB, PI and EI collectively form the second level of Pareto dominance, i.e. $\preceq_{F_2}$, that acts within the $min(\vb*X,\preceq_{F1})$. The associated vector of objective functions to be minimized is:
\begin{equation}
\label{eq8}
F_2 = [LCB(\vb*x), -PI(\vb*x), -EI(\vb*x)].
\end{equation}

{\SetAlgoNoLine%
 \begin{algorithm}[t]
\DontPrintSemicolon
 \BlankLine
  \KwInput{Initial sample size, $S_{in}$; Evaluation budget, $S_{tot}$}
  \BlankLine
  Generate $S_{in}$ samples of the design variable, $\vb*x$, using Latin hypercube sampling: $x^{(1)}, x^{(2)},...x^{(S_{in})}$\;
  Generate the initial database: $D_{0} \leftarrow \{(\vb*x^{(k)}, f(\vb*x)^{(k)}, c_i(\vb*x)^{(k)} \forall  i \in 1,2,...C)\}_{k=1}^{S_{in}}$\\
  \For{$j\gets1$ \KwTo $S_{tot}-S_{in}$}{
    Build GP models of $f$ and $c_i$ $\forall i$ based on dataset $D_{j-1}$.\;
    Generate the optimal solution set $\tilde{\vb*X}_{optimal}$ w.r.t. $\preceq_{F_{1,2}}$ $\rightarrow$ Refer Alg. 2.\;
    Pick a candidate solution, $\vb*x_j$, from $\tilde{\vb*X}_{optimal}$.\;
    Evaluate $\vb*x_j$ to get $f(\vb*x_j), c_i(\vb*x_j) \forall  i \in 1,2,...C$ \;
    $D_{j} \leftarrow \{D_{j-1}, (\vb*x_j, f(\vb*x_j), c_i(\vb*x_j) \forall  i \in 1,2,...C)\}$\;
    }
\textbf{return:} Best solution found
\caption{Pseudocode of TRACE.}
\label{alg1}
\end{algorithm}}
{\SetAlgoNoLine%
 \begin{algorithm}[t]
\DontPrintSemicolon
 \BlankLine
  \KwInput{Population size; Maximum iteration, $M_{iter}$; trained GPs}
  \BlankLine
  Generate initial solution population randomly in $\vb*X$.\;
   \For{$iter\gets1$ \KwTo $M_{iter}$}{
    For each solution, evaluate $f_{cv1}$, $f_{cv2}$, $LCB$, $PI$, $EI$.\;
    Compute nondomination ranks of all solutions w.r.t. $\preceq_{F_{1}}$ where $F_1$ is given by (\ref{eq7}).\; 
    Compute non-domination ranks of all solutions w.r.t. $\preceq_{F_{2}}$ where $F_2$ is given by (\ref{eq8}).\;
    Determine the multi-dominance rank of solutions w.r.t. $\preceq_{F_{1,2}}$, using the binary relation (\ref{eq4}).\;
    Apply multi-dominance rank-based solution selection and stochastic variation of the selected solutions to generate a new solution population.\;
    }
\textbf{return:} Optimal solution set w.r.t. $\preceq_{F_{1,2}}$
\caption{Generate $\tilde{\vb*X}_{optimal}$.}
\label{alg2}
\end{algorithm}}

A pseudocode of TRACE is given in Algorithms \ref{alg1} and \ref{alg2}. To solve the multi-dominance optimization of acquisition functions (in line 5 of Algorithm \ref{alg1}), candidate solutions first undergo separate ranking with respect to both binary relations, $\preceq_{F1}$ and $\preceq_{F2}$, using the non-dominated sorting method \cite{35}. Using these ranks, the best subset of solutions with respect to the combined binary relation $\preceq_{F_{1,2}}$ is iteratively evolved based on (\ref{eq4}). An outline of the evolutionary process is given in Algorithm \ref{alg2}. Once we arrive at a representative subset of candidate solutions that fulfil $\tilde{\vb*X}_{optimal}$ (returned by Algorithm \ref{alg2}), a single solution is randomly picked from the set for evaluation in the next iteration of TRACE. Note that a larger batch size could also be employed by picking multiple solutions from the set for evaluation. As a particular implementation of the selection and variation operators in line 7 of Algorithm 2, we have used techniques from multi-objective particle swarm optimization (MOPSO) \cite{7}; other preferred evolutionary optimizers may also be used here with little change to the overall algorithmic framework.
\begin{theorem} Let $\vb*x\in \vb*X$ and $\mu_{ci}(\vb*x) = c_i(x)$ $\forall i$. Then, the Pareto set expressed as $min(\vb*X,\preceq_{F1})$, where $ F_1 = [f_{cv1}(\vb*x), f_{cv2}(\vb*x)]$, has a one-to-one correspondence to the entire feasible region.
\end{theorem}
\begin{proof}
  It suffices to show that (a) all feasible solutions are non-dominated with respect to each other, and (b) any infeasible solution must be dominated by some feasible solution. For any $\vb*x\in\vb*X$, $f_{cv1}(\vb*x)$ selects some $\mu_{cp}(\vb*x)$ with maximum constraint function value, while $f_{cv2}(\vb*x)$ selects $\mu_{cq}(\vb*x)$ with minimum absolute constraint function value. If $\vb*x$ is feasible, $\mu_{ci}(\vb*x) \leq 0$ $\forall$ $i$ and therefore $p = q$. This implies $f_{cv1}(\vb*x) = - f_{cv2}(\vb*x)$. 
  
  We first prove part (a). Consider any two feasible solutions, $\vb*x^{arb1}$ and $\vb*x^{arb2}$. If $f_{cv1}(\vb*x^{arb1}) < f_{cv1}(\vb*x^{arb2})$, then $f_{cv2}(\vb*x^{arb1}) > f_{cv2}(\vb*x^{arb2})$. Similarly, if $f_{cv1}(\vb*x^{arb1}) > f_{cv1}(\vb*x^{arb2})$, then $f_{cv2}(\vb*x^{arb1}) < f_{cv2}(\vb*x^{arb2})$. This implies that both $\vb*x^{arb1}$ and $\vb*x^{arb2}$ must be non-dominated with respect to each other.
  
  Next, we prove part (b). For any infeasible solution $\vb*x^{arb}$, $f_{cv1}(\vb*x^{arb}) > 0$ and $f_{cv2}(\vb*x^{arb}) \geq 0$. For any feasible solution, $\vb*x^{arb*}$, that lies on a constraint boundary, $f_{cv1}(\vb*x^{arb*}) = f_{cv2}(\vb*x^{arb*}) = 0$. Hence, $\vb*x^{arb*}$ Pareto dominates $\vb*x^{arb}$ with respect to $F_1$. 
\end{proof}

The proposition makes a theoretical assumption of perfect GP predictions, but this is seldom true in practice. In order to account for GP's predictive errors and uncertainty, the definitions of $f_{cv1}$ and $f_{cv2}$ are therefore slightly modified to promote TRACE's exploration for feasible subspaces. Specifically, in (\ref{eq5}) and (\ref{eq6}), $\mu_{ci}(\vb*x)$ is substituted by $\mu_{ci}(\vb*x)-\alpha \sigma_{ci}(\vb*x)$, where $\sigma_{ci}(\vb*x)$ represents the standard deviation of the GP on the $i^{\text{th}}$ constraint and $\alpha$ is a parameter for controlling exploration. In this study, $\alpha$ is set as 0.2.
\section{Simulation results}
\begin{figure}[!t]
    \centering
    \includegraphics[width=0.32\textwidth]{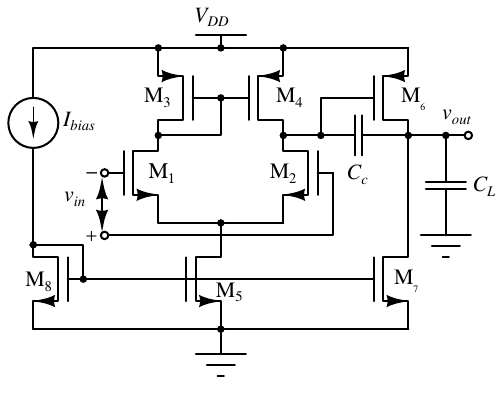}
    \caption{Schematic of the TSMCOA used in our simulations.}
    \label{fig1} 
  \end{figure}
TRACE is validated in a TSMCOA, shown in Fig.~\ref{fig1}, in \SI{65}{\nano\meter} technology with a supply voltage of \SI{1.1}{\volt}. The load capacitance, $C_L$, and compensation capacitance, $C_c$, is taken as \SI{200}{\femto\farad} and \SI{60}{\femto\farad}, respectively. TRACE algorithm is implemented in Python, with Ngspice as the circuit simulator. We examine the effectiveness of TRACE in two different optimization problems, as reported in \cite{8} and \cite{6}. For MOPSO used in the multi-dominance optimization of acquisition functions, the population size and maximum iterations are set as 20 and 100, respectively.

TRACE has been compared with MACE \cite{6} and MPSO \cite{8}. We have also compared the results of TRACE with two different constrained acquisition ensembles of EIPF (which implements a constrained acquisition function of EI$\times$PF) and PIPF (which implements a constrained acquisition function of PI$\times$PF) \cite{s1}. For both the test cases, the design vector, $\vb* x$, consists of 11 design variables, including the lengths and widths of the transistors and the biasing current. The bounds for the design variables are given in Table \ref{table5}.

The corresponding bounds for the design variables are given in Table~\ref{table5}.
   \begin{table}[!t]
    \begin{center}
    \caption{Bounds for the design variables.}
    \label{table5}
    \scalebox{1}{
      \begin{tabular}[!t]{cccc}
        \toprule
        {\bf Test Case} & Width ($W$)&  Length ($L$) & $I_{bias}$\\
        \midrule
        Gain Maximisation&[120, 3600] \si{\nano\meter}&[60, 360] \si{\nano\meter}& [10,100] \si{\micro\ampere}\\
        Area Minimisation&[120, 1200] \si{\nano\meter}&[60, 120] \si{\nano\meter}& [10, 40] \si{\micro\ampere}\\
        \bottomrule
      \end{tabular}}
    \end{center}
  \end{table}
  
In the first test case, we have implemented the gain maximization problem reported in \cite{6} in TSMCOA in \SI{65}{\nano\meter} technology. The evaluation budget is set as 200, with an initial sample size of 20 for all the algorithms. The optimization problem is framed as follows:
\begin{equation*}
\begin{aligned}
& \underset{\vb*x}{\max}
&& \text{Gain }(A_v) \\
& \text{subject to}
& & \text{Phase margin }(PM)\ge\SI{60}{\degree} \\
&&& \text{Unity gain bandwidth }(UGB) \ge\SI{200}{\mega\hertz} \\
\end{aligned}
\end{equation*}

Area minimization of TSMCOA, as reported in \cite{8}, is considered as the second test case. The evaluation budget is set as 300, with an initial sample size of 70 for all the algorithms. The optimization problem is framed as follows:
\begin{equation*}
\begin{aligned}
& \underset{\vb*x}{\text{min}}
&& f(\vb*{x}) = \Sigma_{i=1}^{8}W_i\times L_i\\
& \text{subject to}
& & \text{Voltage gain } (A_v) \ge\SI{20}{\decibel} \\
&&& \text{Slew Rate }(SR) \ge\SI{100}{\volt\per\micro\s} \\
&&& \text{Cut-off frequency }(f_{3dB}) \ge\SI{10}{\mega\hertz} \\
&&& \text{Unity gain bandwidth }(UGB) \ge\SI{100}{\mega\hertz} \\
&&& \text{Phase margin }(PM)\ge\SI{60}{\degree} \\
&&& \text{Noise}(S_n(f))\le\SI{60}{\nano\volt\per\sqrt{Hz}}\text{ at } \SI{1}{\mega\hertz}\\
&&& \SI{0.7}{\volt} \leq \text{Input common mode voltage} \leq \SI{0.8}{\volt}.
\end{aligned}
\end{equation*}
 where $W_i$ and $L_i$ are the width and length of the $i^\text{th}$ transistor.

   \begin{table}[!t]
    \caption{Simulation results of gain and area optimization of TSMCOA.\vspace{-0.4cm}}
    \begin{center}
      \label{table1}
      \centering
      \scalebox{0.9}{
      \begin{tabular}[!t]{lcccc|cccc}
        \toprule
        \multirow{2}{*}{\textbf{Algorithm}} & \multicolumn{4}{c}{{\bf Gain} (\si{\decibel})} & \multicolumn{4}{c}{{\bf Area} (\si{\micro\meter^2})}\\
        \cmidrule{2-9}
        &{\bf Best}&{\bf Worst}& {\bf Mean}& {\bf Std}&{\bf Best}&{\bf Worst}& {\bf Mean}& {\bf Std}\\
        \midrule
        MPSO\cite{8}&13.2&4.2&6.2&4.9&0.2750&0.3250&0.3074&0.03\\
        PIPF\cite{s1}&14.2&3.4&5.2&3.9&0.2549&0.2923&0.2808&0.03\\
        EIPF\cite{s1}&12.2&2.6&8.4&4.4&0.2649&0.3012&0.2921&0.02\\
        MACE\cite{6}& 24.0&10.1&14.7&5.1&0.1925&0.2441&0.2186&0.02\\
        TRACE&41.4&30.6&34.7&3.2&0.1235&0.1897&0.1796&0.02\\
        \bottomrule
      \end{tabular} }
    \end{center}
  \end{table}

  Table~\ref{table1} compares the best, worst, mean, and standard deviation (Std) for $10$ consecutive runs by all the algorithms for both the test cases. The results show that TRACE achieves a better value for gain as well as area compared to all the other algorithms. The mean convergence plots for 10 runs by TRACE and MACE for the two test cases are given in Fig.~\ref{fig2} and Fig.~\ref{fig3}, respectively. TRACE has demonstrated consistent performance in terms of convergence to better optimal designs in the two test cases.
  
  For both the test cases considered in this study, about 99\% of the candidate solutions selected by MPSO, PIPF, and EIPF violated the constraints, with MACE reporting a constraint violation of 89\% and 93\%, respectively. With 61\% and 83\% constraint violations in the two test cases, TRACE showed substantially fewer constraint violations by the candidate solutions compared to other algorithms. This demonstrates the efficacy of TRACE in extremely limited design spaces, which is typically the case in analog circuit design when strict circuit requirements are taken into account. 

  The design parameters and the corresponding circuit specifications of the best solution obtained by TRACE are given in Table~\ref{table2} for gain optimization and in Table~\ref{table3} and Table~\ref{table4} for area optimization. Table~\ref{table4} summarizes the results for the Monte Carlo (MC) simulations for the optimal design for 1000 runs across various PVT corners in the area optimization test case. The MC results validate the performance of the optimal design for TSMCOA obtained by TRACE.
  \begin{table}[!t]
    \begin{center} 
    \caption{Optimum parameters and specifications obtained by TRACE for gain optimization.\vspace{-0.4cm}}
    \label{table2}
    \scalebox{0.9}{
      \begin{tabular}[!t]{ccccccc}
        \toprule
        $I_{bias}$ & $W_{1,2}$ & $W_{3,4}$&  $W_{5,8}$ & $W_6$& $W_7$ & $L_{1,2}$ \\
        \SI{29.8}{\micro\ampere}&\SI{336}{\nano\meter}&\SI{643}{\nano\meter}&\SI{769}{\nano\meter}&\SI{3167}{\nano\meter}&\SI{2508}{\nano\meter} & \SI{190}{\nano\meter}\\
        \midrule
        $L_{3,4}$&  $L_{5,8}$ & $L_6$& $L_7$ & $A_v$&$PM$&$UGB$\\
        \SI{212}{\nano\meter}&\SI{288}{\nano\meter}&\SI{95}{\nano\meter}&\SI{252}{\nano\meter}&\SI{41}{\decibel}&61$^{\circ}$&\SI{268}{\mega\hertz}\\
        \bottomrule
      \end{tabular} }
    \end{center}
  \end{table}
 \begin{table}[!t]
    \begin{center}
    \caption{Optimum parameters obtained by TRACE for area optimization.\vspace{-0.1cm}}
    \label{table3}
    \scalebox{0.95}{
      \begin{tabular}[!t]{ccccccc}
        \toprule
        $I_{bias}$ & $W_{1,2}$ & $W_{3,4}$&  $W_{5,8}$ & $W_6$& $W_7$\\
        \SI{24.1}{\micro\ampere}&\SI{205}{\nano\meter}&\SI{120}{\nano\meter}&\SI{120}{\nano\meter}&\SI{498}{\nano\meter}&\SI{403}{\nano\meter} \\
        \midrule
        $L_{1,2}$ & $L_{3,4}$&  $L_{5,8}$ & $L_6$& $L_7$ &\\
        \SI{74}{\nano\meter}&\SI{93}{\nano\meter}&\SI{70}{\nano\meter}&\SI{60}{\nano\meter}&\SI{60}{\nano\meter}&\\
        \bottomrule
      \end{tabular}}
    \end{center}
  \end{table}
\begin{table}[!t]
    \caption{Specifications obtained for TSMCOA in area optimization.}
    \label{table4}
    \begin{center}
    \scalebox{0.9}{
      \begin{tabular}{lccc}
        \toprule
        \multirow{2}{*}{\textbf{Design Criteria}}&{\bf Specifi-}&\multicolumn{2}{c}{\textbf{MC Results}} \\
        \cline{3-4}
        \textbf{}&{\bf cations}&\textbf{Mean}&\textbf{Std} \\
        \midrule
        $A_v$ (\si{\decibel})&$\geq20$&21.4&2.17\\
        $f_{3dB}$ (\si{\mega\hertz})&$\geq10$&32.6&24.5\\
        $UGB$ (\si{\mega\hertz})&$\geq100$&128.1&50.65\\
        Phase Margin $(^{\circ})$&$\geq60$&62.8&4.40 \\
        $SR$ (\si{\volt\per\micro\second})&$\geq100$&290&53  \\
        $S_n(f)$@1MHz (\si{\nano\volt\per\sqrt{Hz}})&$\leq60$&57&0\\
        $S_n(f)$@10MHz (\si{\nano\volt\per\sqrt{Hz}})&-&20&0\\
        $Power$ (\si{\micro\watt})&-&79&21\\
        $CMRR$ (\si{\decibel})&-&38.6&10.4\\
        $PSRR+$ (\si{\decibel})&-&20.8&2.0\\
        $PSRR-$ (\si{\decibel})&-&45.4&10.7\\
        Settling time with 2\% tol. (\si{\nano\second})&-&5.4&0.82\\ 
        Settling time with 5\% tol. (\si{\nano\second})&-&4.3&0.54\\ 
        $A$ (\si{\micro\meter^{2}})&-& \multicolumn{2}{c}{\textbf{0.1235}} \\
        \bottomrule
      \end{tabular} }
    \end{center}
  \end{table}
     
   \begin{figure}[!t]
    \centering
    \includegraphics[width=0.43\textwidth]{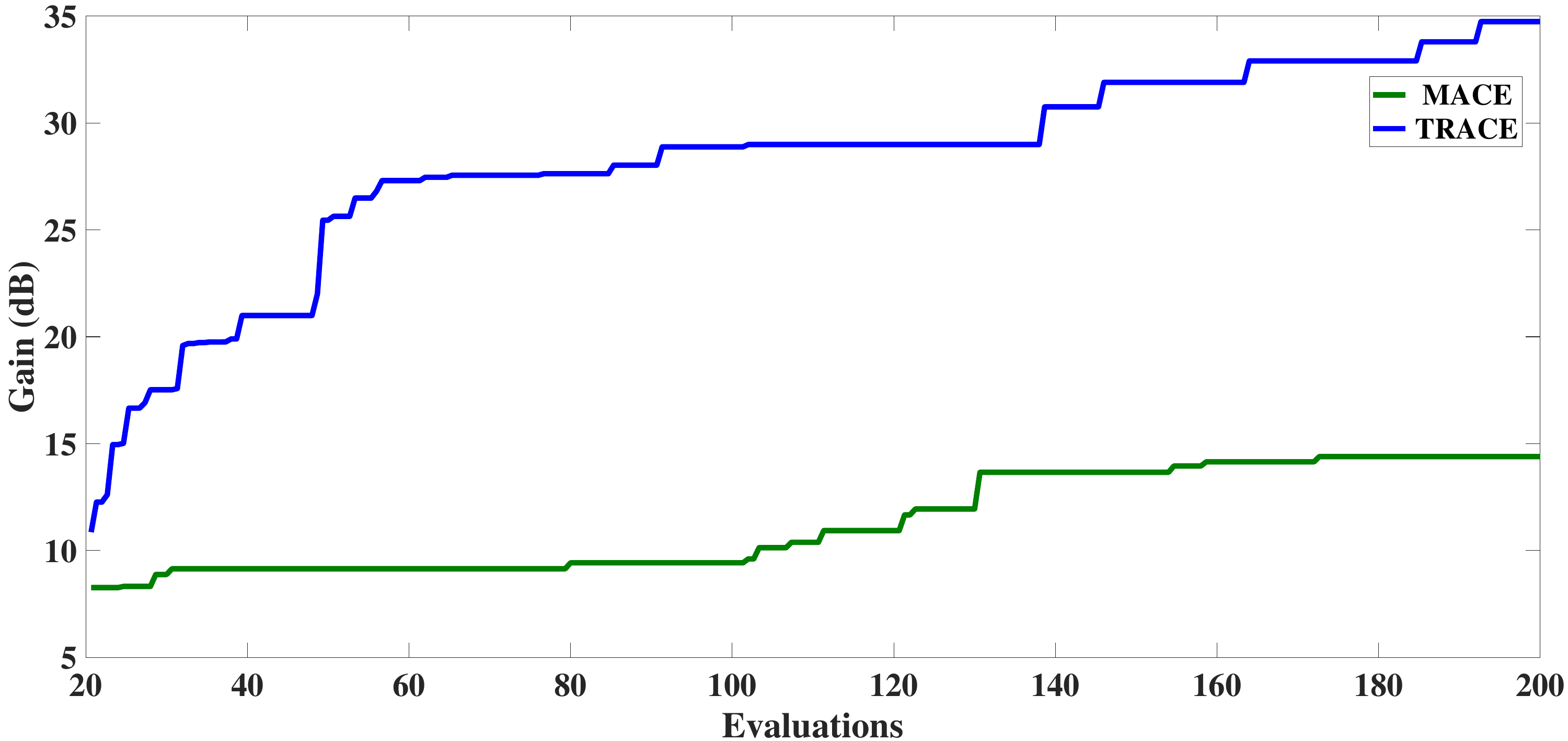}
    \caption{Comparison of mean convergence plots of MACE and TRACE for gain optimization.}
    \label{fig2} 
  \end{figure}
  \begin{figure}[!t]
    \centering
    \includegraphics[width=0.43\textwidth]{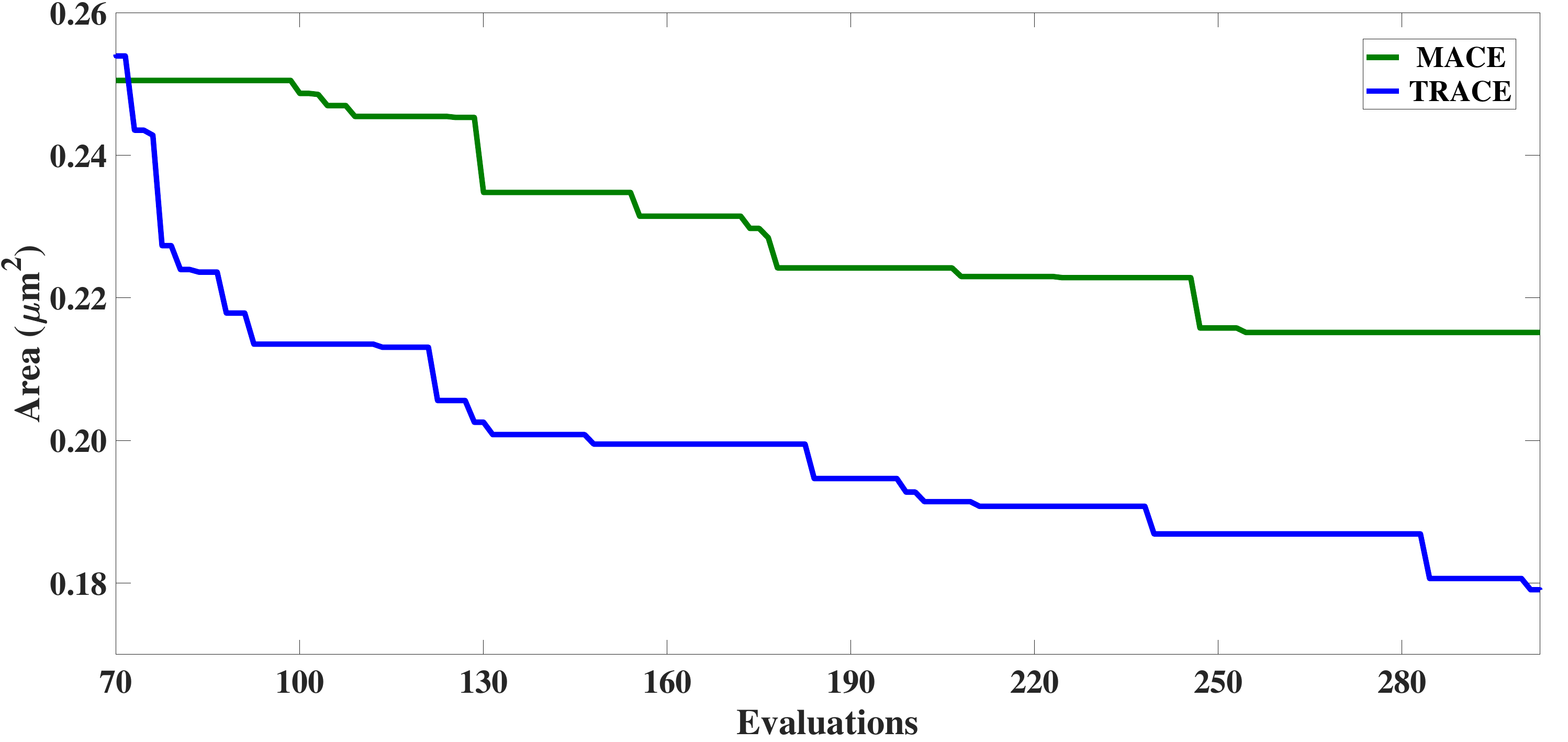}
    \caption{Comparison of mean convergence plots of MACE and TRACE for area optimization.}
    \label{fig3} 
  \end{figure}

  A comparison of TRACE with ANN-PSO reported in \cite{23} for the area optimisation of TSMCOA in \SI{65}{\nano\meter} is also performed. For a fair comparison, TRACE is run with same constraints as reported in \cite{23}. The comparison results are presented in Table~\ref{tables4} and the optimum parameters obtained by TRACE is given in Table~\ref{tables5}. The results show that TRACE is able to achieve a 29\% reduction in area compared to \cite{23}.

\begin{table}[!t]
    \caption{Specifications obtained for TSMCOA in area optimization.}
    \label{tables4}
    \begin{center}
      \begin{tabular}{lccc}
        \toprule
        \textbf{Design Criteria}&{\bf Specs.}&\textbf{ANN-PSO} \cite{23}&\textbf{This work} \\
        \midrule
        $A_v$ (\si{\decibel})&$\geq20$&20.2&21.5\\
        $UGB$ (\si{\mega\hertz})&$\geq100$&106.4&107.4\\
        Phase Margin $(^{\circ})$&$\geq60$&61&63.4 \\
        $SR$ (\si{\volt\per\micro\second})&$\geq100$&176&221\\
        $Power$ (\si{\micro\watt})&-&79&94\\
        $CMRR$ (\si{\decibel})&-&30.3&27.1\\
        $PSRR+$ (\si{\decibel})&-&15.64&11.0\\
        $A$ (\si{\micro\meter^{2}})&-&{\textbf{0.121}}&{\textbf{0.086}} \\
        \midrule
      \end{tabular} 
    \end{center}
  \end{table}

   \begin{table}[!t]
    \begin{center}
    \caption{Optimum parameters obtained by TRACE for area optimization with the same consrtaints as reported in \cite{23}.}
    \label{tables5}
    \scalebox{0.95}{
      \begin{tabular}[!t]{ccccccc}
        \toprule
        $I_{bias}$ & $W_{1,2}$ & $W_{3,4}$&  $W_{5,8}$ & $W_6$& $W_7$ & $L_{1-8}$\\
        \SI{30.8}{\micro\ampere}&\SI{120}{\nano\meter}&\SI{120}{\nano\meter}&\SI{120}{\nano\meter}&\SI{588}{\nano\meter}&\SI{120}{\nano\meter} & \SI{60}{\nano\meter}\\
        \bottomrule
      \end{tabular}}
    \end{center}
  \end{table}

  \section{Test Functions}
TRACE algorithm has been validated using three analytical benchmark functions: test function 1 \cite{s1}, test function 2 \cite{s3} and branin function \cite{s2}. TRACE is compared with PSO, EIPF, PIPF and MACE. In all the cases, TRACE has performed significantly better in terms of fewer constraint violations and better accuracy. For all the test cases, the best, worst, mean, and standard deviation along with the average percentage of constraint violations (CV) by the candidate solutions for ten consecutive runs are given in Tables \ref{tables1}, \ref{tables2} and \ref{tables3}.

\subsection{Test Function 1:}

\begin{equation*}
\begin{aligned}
& \underset{\vb*x}{\min}
&& f(x) = \text{cos}(2x_1)\text{cos}(x_2)+\text{sin}(x)_1 \\
& \text{subject to}\\
&&& \text{cos}(x_1)\text{cos}(x_2)-\text{sin}(x_1)\text{sin}(x_2)-0.5\leq0\\
&&& x_1\;\text{,}\;x_2\;\in\;[0,6]
\end{aligned}
\end{equation*}

  \begin{table}[hbt!]
    \caption{Simulation results of Test Function 1.}
    \begin{center}
      \label{tables1}
      \centering
      \scalebox{1}{
      \begin{tabular}[]{lccccc}
        \toprule
        \bf Algorithm &{\bf Best}&{\bf Worst}& {\bf Mean}& {\bf Std}& {\bf CV (\%)}\\
        \midrule
        PSO&-1.9994&-1.1223&-1.7209&0.4130&39\\
        PIPF&-1.9951&-1.7704&-1.9431&0.0632&29\\
        EIPF&-1.9997&-1.6850&-1.9090&0.1144&33\\
        MACE&-1.9985&-1.9507&-1.9792&0.0213&13\\
        TRACE&-1.9990&-1.9807&-1.9917&0.0067&11\\
        \bottomrule
      \end{tabular}}
      \vspace{0.1cm}
      \footnotesize{\\Evaluation budget = 50, Initial sample size = 10}
    \end{center}
  \end{table}

\subsection{Test Function 2:}

\begin{equation*}
\begin{aligned}
& \underset{\vb*x}{\max}
&& f(x) = (x_1-1)^2 + (x_2-0.5)^2 \\
& \text{subject to}\\
&&& \big((x_1-3)^2+(x_2+2)^2\big)e^{-x_2^7}-12\leq0\\
&&& 10x_1+x_2-7\leq0\\
&&& (x_1-0.5)^2+(x_2-0.5)^2-0.2\leq0\\
&&& x_1\;\text{,}\;x_2\;\;\in\;[0,1]\\
\end{aligned}
\end{equation*}

  \begin{table}[hbt!]
    \caption{Simulation results of Test Function 2.}
    \begin{center}
      \label{tables2}
      \centering
      \scalebox{1}{
      \begin{tabular}[hbt!]{lccccc}
        \toprule
        \bf Algorithm &{\bf Best}&{\bf Worst}& {\bf Mean}& {\bf Std}& {\bf CV (\%)}\\
        \midrule
        PSO&-&-&-&-&100\\
        PIPF&-&-&-&-&100\\
        EIPF&-&-&-&-&100\\
        MACE&0.6858&0.6286&0.6639&0.0193&57\\
        TRACE&0.7332&0.6437&0.6976&0.0312&43\\
        \bottomrule
      \end{tabular}}
      \vspace{0.1cm}
      \footnotesize{\\Evaluation budget = 160, Initial sample size = 30}
    \end{center}
  \end{table}

\subsection{Branin Function:}

\begin{equation*}
\begin{aligned}
& \underset{\vb*x}{\max}
&& f(x) = (x_1-10)^2 + (x_2-15)^2 \\
& \text{subject to}\\
&&& (x_2-\frac{5.1}{4\pi^2}x_1^2+\frac{5}{\pi}x_1-6)^2 +\\
&&& \;\;\;\;\;\;\;\;\;\;\;\;10\big(1-\frac{8}{\pi}\big)\text{cos}(x_1)+5\leq0\\
&&& x_1\;\in\;[-5,10]\\
&&& x_2\;\in\;[0,15]\\
\end{aligned}
\end{equation*}

  \begin{table}[hbt!]
    \caption{Simulation results of Branin Function.}
    \begin{center}
      \label{tables3}
      \centering
      \scalebox{1}{
      \begin{tabular}[hbt!]{lcccccc}
        \toprule
        \bf Algorithm &{\bf Best}&{\bf Worst}& {\bf Mean}& {\bf Std}& {\bf CV (\%)}\\
        \midrule
        PSO&268.8&188.9&239.1&34.3&80\\
        PIPF&-&-&-&-&100\\
        EIPF&-&-&-&-&100\\
        MACE&266.6&206.7&226.5&20.9&34\\
        TRACE&266.9&225.1&255.9&16.1&30\\
        \bottomrule
      \end{tabular}}
      \vspace{0.1cm}
      \footnotesize{\\Evaluation budget = 200, Initial sample size = 30}
    \end{center}
  \end{table}

\section{Conclusion}
Solutions recommended by existing Bayesian optimization (BO) algorithms, when applied to analog circuits, often fail to satisfy tight circuit specifications. In this paper, a novel BO with a tiered acquisition function ensemble (TRACE) is proposed to address this issue. TRACE is unique in the sense that it provably prioritizes the discovery of a feasible solution set (see Proposition 1), within which the true optimum can be more effectively located. The methodology has been validated in a two-stage Miller compensated operational amplifier in area and gain optimization in a \SI{65}{\nano\meter} technology, with up to 43\% and 33\% improvement in values of gain and area, respectively, when compared with the state-of-the-art algorithms, with significantly lesser constraint violations. This demonstrates the considerable application potential of the proposed approach in the automation of analog circuit design.

\bibliographystyle{IEEEtran}
\bibliography{references}
\end{document}